\documentclass[conference]{IEEEtran}   	
\usepackage{geometry}                		
\usepackage[parfill]{parskip}    		
\usepackage{graphicx}				
		
\usepackage{amssymb}
\usepackage{amsthm,amsmath}
\usepackage{algorithmic,algorithm,lipsum}
\usepackage[usenames, dvipsnames]{color}
\usepackage{wrapfig}

\theoremstyle{definition}
\newtheorem{lemma}{Lemma}[section]
\newtheorem{definition}{Definition}[section]

\usepackage{xspace} 

\newcommand{\G}[1]{\mathbf{\mathcal{#1}}}

\title{Algorithms for Item Categorization Based on Ordinal Ranking Data}
\author{Josh Girson, Shuchin Aeron \\ Tufts University, Medford, MA}
\date{}							

\begin{document}
\maketitle

\begin{abstract}
We present a new method for identifying the latent categorization of items based on their rankings. Complimenting a recent work that uses a Dirichlet prior on preference vectors and variational inference, we show that this problem can be effectively dealt with using existing community detection algorithms, with the communities corresponding to item categories. In particular we convert the bipartite ranking data to a unipartite graph of item affinities, and apply community detection algorithms. In this context we modify an existing algorithm - namely the label propagation algorithm to a variant that uses the distance between the nodes for weighting the label propagation - to identify the categories.
We propose and analyze a synthetic ordinal ranking model and show its relation to the recently much studied stochastic block model. We test our algorithms on synthetic data and compare performance with several popular community detection algorithms. We also test the method on real data sets of movie categorization from the Movie Lens database. In all of the cases our algorithm is able to identify the categories for a suitable choice of tuning parameter.


\end{abstract}

\section{Introduction}
In this paper we consider the problem of item categorization based on choice, preference or ranking data by a number of voters. So far using the voter-rating matrix, the literature has focused on voter categorization instead of item categorization \cite{Beigi_2014}. We are partly motivated by a recent work, \cite{Agarwal_IJCAI2016} where item categorization based on choice statistics was considered. Using a Dirichlet prior on the preferences for each user coupled with a random utility model for making choices, the authors use a variational algorithm to infer the categories. In contrast to these approaches, in this paper a new ranking (choice) model among the categories is presented followed by the use of community detection algorithms \cite{Fortunato,GNM} for category discovery by converting the bipartite graph of ratings to the unipartite similarity graph of item similarities.

In this context our contribution is two-fold - (a) We analyze the expected connectivity in the similarity graph and link our model to the recently much studied stochastic block model. This implies that one can understand information theoretic limits on the discovery of categories by directly using recent results in \cite{IT_SBM_Abbe}\footnote{We note that this conversion can also be applied to other latent category choice/ranking models and similarity measures, but the ensuing analysis, at present, seems quite complicated.}; (b) We propose a variation of the standard label propagation algorithm \cite{2007_Label_Prop}, which we refer to as the weighted label propagation algorithm where the path distance in label aggregation is incorporated to avoid false inclusions.  This can be thought of as nodes having more influence on their neighbors than distant nodes in the graph, which slows the ability of one label to consume an entire graph. We show that this simple modification is useful in avoiding misclustering in some cases where the standard label propagation algorithm fails.  We also compare the performance of this algorithm with other algorithms such as Clauset-Neuman-Moore (CNM) \cite{CNM} algorithm that maximizes the modularity metric for clustering. Synthetic data experiments show that the weighted label propagation algorithm is competitive with these existing methods.

The rest of the paper is organized as follows. In the following section we introduce the generative ranking model used in the analysis of the algorithm and the relation to the standard block model is explored through the analysis of the chosen similarity function. In section 3, we present a modification of the label propagation algorithm that introduces a weighting of the labels that are to be selected at each round. The new algorithm aims to give larger weight to labels that are closer to their source vertices and less strength to labels that are farther. Finally, in section 4 we present the experimental results of the algorithm, both on the synthetic data produced by the generative model and on real data from the Movie Lens database \cite{MovieLens}. 

\section{Generation of a Synthetic Rating Model}

Algorithm \ref{alg:SynthRankGen} outlines the model used to generate the ranking data among categories. The parameters for this data generation are $C$: the number of categories; $S$: the number of items in each category; $p$: the expected number of items each category will swap with every other category (referred to as the mixing parameter) to allow for choice variability; and $V$: the number of voters. 

The main idea behind the synthetic ranking generation algorithm (Algorithm \ref{alg:SynthRankGen}) is as follows. For each voter, the following process is repeated. Each of the categories are ordered randomly from $1$ to $C$. Then, each of the $S$ items of the $C$ categories are given a unique random integer value in the range $[C_i S, \: (C_i + 1) S)$ where $C_i$ is the location of the category in the ordering. This process produces an ordinal ranking system of the $N = S C$ items with items in the same category initially given close proximity to one another. 

\begin{algorithm}[h]
\caption{Synthetic Ranking Generation}
\begin{algorithmic}
\label{alg:SynthRankGen}
\STATE \textbf{Input:} $\mathbf{S}$: Category size; $\mathbf{C}$: Number of Categories; $\mathbf{p}$: Mixing parameter; $\mathbf{V}$: Number of voters
\FORALL{voters in range(V)}
        \STATE $\mathbf{\mathcal{O}} \leftarrow$ random permutation of the $C$ categories
        \STATE $i \leftarrow 1$
	\FORALL{$ j \in \mathbf{\mathcal{O}}$}
		\STATE $R_i \leftarrow$ random permutation of $[jS, (j+1)S - 1]$
		\STATE $i \leftarrow i + 1$
	\ENDFOR
	\STATE $\mathbf{\mathcal{O}_2} \leftarrow$ random permutation of the $C$ categories
	\FOR{category $c_1 \in \mathbf{\mathcal{O}_2}$}
		\FORALL{other categories $c_2 \in \mathbf{\mathcal{O}_2} - \{c1\}$}
			\STATE Swap $p$ random unique items between $c_1$ and $c_2$
		\ENDFOR
	\ENDFOR
	\STATE $R \leftarrow$ collapse $[R_1, \dots, R_C]$
\ENDFOR
\end{algorithmic}
\end{algorithm}

Once all of the items are placed as such, \emph{the mixing process begins}. In a new random order, generated separately from the ordering above, the $C$ categories each swap $p$ items with each of the other categories. These $p$ items are chosen randomly and uniformly from the $S$ items of the category. It is worth noting that if 3 categories swap in the order $A$, $B$ $C$, then the item that swapped into $B$ from $A$ could be the same ones that $B$ swaps with $C$. This is to say that the $p$ items swapped from a category do not have to have originated in that category.

\begin{figure}[b] 
    \centering
    \includegraphics[width=0.25\textwidth]{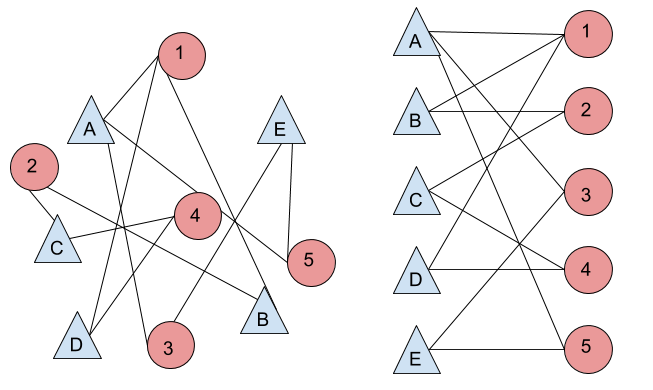}
    \caption{Two views of the same bipartite graph. On the right side, the two groups are clearly split, while on the left it is more difficult to see the separation}
    \label{fig:bipartite_graph}
\end{figure}

As mentioned, the above ranking generation is repeated for each of the $V$ voters to create a complete dataset of \emph{ordinal} rankings. Once this is completed, the bipartite ranking data is converted into a unipartite graph using Algorithm \ref{alg:BipartiteToUnipartite}. This algorithm is based around the idea of collapsing the edges between voters and items into direct edges between items.  For each voter, call it $v$, each pair of item is compared using a similarity function, $S$. The values of this similarity function affect whether or not there is to be an edge between these two items.

As edges in the final graph are to represent a strong relationship between elements, the similarity function should have a higher value for items that are rated similarly and a lower rating for those with a larger difference. The reasoning behind this is that users typically prefer similar things and will thus give similar ratings to items of the same category while items of different categories should be rated differently due to preference. There are a variety of functions that can be chosen such as the Cosine Similarity Measure \cite{Beigi_2014} or the Pearson Correlation Coefficient \cite{2013arXiv1311.1924M}. In this paper, a different function, introduced below, is used that is more amenable to analysis as shown in the next section.

\begin{algorithm}[h]
\caption{Conversion of Bipartite Ranking Graph to Unipartite}
\begin{algorithmic}
\label{alg:BipartiteToUnipartite}
\STATE \textbf{Input:} An $M \times N$ matrix $\mathbf{B}$ where the columns are the elements and the rows are the ratings of the voters. A similarity function $S$ that maps pairs of rankings to a weight. A threshold, $\epsilon$, to cut off the weights.
\STATE \textbf{Initialize:} A weight matrix, $\mathbf{W}$, of all zeros of size $N \times N$. An empty graph $\mathcal{G}$ 
\FOR{Voter v in range(M)}
\FOR{i in range(N)}
\FOR{j in range(i + 1, N)}
\STATE $\mathbf{W}[i][j] += S(\mathbf{B}[v][i], \mathbf{B}[v][j])$
\ENDFOR
\ENDFOR
\ENDFOR

\FOR{i in range(N)}
\FOR{j in range(N)}
\IF{$\mathbf{W}[i][j] > \epsilon$}
\STATE Add edge $e_{ij}$ to $\mathcal{G}$
\ENDIF 
\ENDFOR
\ENDFOR
\end{algorithmic}
\end{algorithm}

\subsection{Relation to the Stochastic Block Model (SBM)}

In this section, we aim to discover the relationship between this model and the stochastic block model (SBM)\footnote{Also equivalent to the planted partition model.} recently used for deriving information theoretic bounds \cite{IT_SBM_Abbe}. In the standard block model, the input parameters are the community size, number of communities, intracommunity density ($p_{in}$ or $\alpha$) and the intercommunity density ($p_{out}$ or $\beta$). The output is a graph of size community size $\times$ number of communities where the chance of an edge between nodes of the same community is $\alpha$ and the chance of an edge between nodes of different communities is $\beta$. It is clear that in this ranking model, $S$ and $C$ are directly representative of the community size and number of communities parameters in the standard block model. It is necessary to examine Algorithm \ref{alg:BipartiteToUnipartite}, which actually introduces the edges, to determine how $\alpha$ and $\beta$ relate to the standard block model. Algorithm \ref{alg:BipartiteToUnipartite} includes a threshold, $\epsilon$, that serves as a cutoff for when an edge will and will not connect two elements, $a$ and $b$. $\alpha$ and $\beta$ are discovered by calculating the probability that the average similarity value between the two elements across all voters is above this threshold.
\\
\begin{definition} \emph{\textbf{Alpha:}} The probability of an edge between elements of the same category in the generative model
\begin{align}
\alpha_{(a,b)} = P \left\{ \frac{\sum\limits_{v \in V} S_v (a,b,N)}{|V|} \geq \epsilon \right\}
\end{align}
\end{definition}

\begin{definition} \emph{\textbf{Beta:}} The probability of an edge between elements of different categories in the generative model
\begin{align}
\beta_{(a,b)} = P \left\{ \frac{\sum\limits_{v \in V} S_v (a,b,N)}{|V|} \geq \epsilon \right\}
\end{align}
\end{definition}

Since these probabilities are implicit, in the following sections we will analyze their expected behaviors. 

\subsection{Analysis of Similarity Function}

As the calculation of $\alpha$ and $\beta$ is an implicit calculation, it is useful to calculate the expected values of the similarity function for elements in the same category and those in different categories as a proxy for examining these measures. We first introduce the similarity function we are using in the algorithm and analysis: \[Sim(a, b, N) = 1 - \frac{|a - b|}{N}\]. As discussed above, this function is chosen as it gives higher weights to elements that are more closely ranked, such as those in the same category, than elements that are ranked very differently. As the threshold, $\epsilon$, should lead to the addition of only edges that represent strong relationships, the expected value of our similarity function is used. In this calculation it is hypothesized that there is an equal chance of choosing any pair of elements $a$ and $b$. 
\\
\begin{lemma} \emph{Expected Value of Similarity Function:}
\begin{align}
\mathbb{E}[|a-b|] = \frac{N+1}{3} \Longrightarrow \mathbb{E}[S(a,b,N)] & = 1 - \frac{N+1}{3N}
\end{align}
\end{lemma}
\begin{proof}
Please see Appendix A
\end{proof}

\subsubsection{Analysis of Similarity Function for Elements of the Same Category}

As mentioned above, as the calculation of $\alpha$ is implicit, it is useful to consider what the expected value for the similarity function would be between elements in the same category. All that must be considered is the value of the distance, as the value of the similarity function is inversely proportional to the distance. The expected distance is calculated by averaging the distance between all possible combinations of elements. 

Situations involving mixing must also be considered. The total distance of all possible pairs of elements can be calculated regardless of where a new element, $X$, swaps in; $X$ can swap in with any of the other elements already in the category. In order to calculate the number of combinations, the $S$ possible locations swapped in the category are multiplied by the $\binom{S}{2}$ different element pairs to yield the total $S * \binom{S}{2}$ possible distance combinations. 

To get the total distance, we recognize that the new element, $X$, will be paired with each other element already in the category $S - 1$ times, every round except that in which $X$ will have swapped with the element itself. Similarly, all of the standard intracategory pairings will still happen $S - 2$ times: the two exceptions being when either of the two elements is the one swapped. Thus we multiply the standard sum of intracategory distances, $\sum\limits_{i = 1}^{S - 1} i (S - i)$ by $S - 2$. These values are combined to get the expected distance expressed below.
\\
\begin{lemma} \emph{Expected Intracategory Distance with 1 swap:}
\begin{align}
\mathbb{E}\left[|a - b|_{same}\right]  = \frac{(S - 1) \sum\limits_{i = 1}^{S} |X - i| + (S - 2) \sum\limits_{i = 1}^{S - 1} i (S - i)}{S * \binom{S}{2}}
\end{align}
\end{lemma}

This idea is extended to the general case of $p$ swaps. Let $X_1, \dots, X_p$ be the $p$ elements that are swapping into the category. 
\\
\begin{lemma} \emph{Expected Intracategory Distance with $p$ swaps}
\begin{align}
& \mathbb{E}\left[|a - b|_{same}\right]  = \nonumber \\ 
&\frac{
\binom{S}{p} \sum\limits_{i = 1}^{p} \sum\limits_{j = i + 1}^{p} |X_i - X_j| + 
\binom{S - 1}{p} \sum\limits_{i = 1}^{S} \sum\limits_{j = 1}^{p} |X_j - i| +
\binom{S - 2}{p} \sum\limits_{i = 1}^{S - 1} i (S - i)
}
{\binom{S}{2} \binom{S}{p}}
\end{align}
\end{lemma}
\begin{proof}
This lemma follows from the extension of the above ideas discussed in the formation of Lemma II.2
\end{proof}

Clearly these two equations are equivalent if the value $1$ is substituted for $p$. One thing worth noting is that when $p = 0$, $\mathbb{E}\left[|a - b|_{same}\right]  = \frac{S+1}{3}$. 

\subsubsection{Analysis of Similarity Function for Elements in Different Categories}

The analysis of the distance function for elements in different categories is slightly more difficult. As categories can have different distances from each other, a new variable, $D$, is introduced to signify this distance. To give an example of this measurement, if two elements are in adjacently ranked categories, $D = 0$, if they are two away, i.e. category 1 and category 3, then $D = 1$. 
\\
\begin{lemma} \emph{Expected Intercategory Distance with 1 swap}
\begin{align}
\mathbb{E}\left[|a - b|_{diff}\right] &= 
\frac{\left(1 + (S - 1)^2\right) S^3 (D + 1) + 4S \sum\limits_{i = 1}^{S - 1} i (S - i)}{S^4}
\end{align}
\end{lemma}
\begin{proof}
Please See Appendix B
\end{proof}

We can also extend this formulation to the situation where there are $p$ elements swapping. As all the elements swapping in come from the other category, there is no need to introduce the $X_i$ variables seen in the intracategory comparison.
\\
\begin{lemma} \emph{Expected Intercategory Distance with $p$ swaps}
\begin{align}
& \mathbb{E}\left[|a - b|_{diff}\right] \nonumber \\ 
& = 
\frac{\left(\binom{S - 1}{p}^2 + \binom{S - 1}{p - 1}^2 \right) S^3 (D + 1) + 4 \binom{S - 2}{p - 1} \binom{S}{p} \sum\limits_{i = 1}^{S - 1} i (S - i)}{S^2 * \binom{S}{p}^2}
\end{align}
\end{lemma}
\begin{proof}
Please See Appendix B
\end{proof}

For the majority of our comparisons, the case where $C = 2$ is considered,  which causes $D$ to always be zero as a greater distance is impossible. By fixing $D = 0$, the terms involving $D$ can be disregarded as they become the multiplicative identity.

\subsubsection{Comparisons}

In order to achieve a graph in which the community structure is discernible, we would like to see that \[\beta < \epsilon < \alpha \] so that edges within a category are more likely to be added than edges joining elements between different categories. Recalling that as a proxy we can instead use the relationship between their expected distances (which is the inverse of the relationship between $\alpha$ and $\beta$), we require $\mathbb{E}\left[|a - b|_{same}\right] < \mathbb{E}[|a - b|] < \mathbb{E}\left[|a - b|_{diff}\right]$. As a base case, we compare the expected distance values without any mixing: 
\[\frac{S+1}{3} \overset{?}{<} \frac{N+1}{3} \overset{?}{<}  S(D+1)\]
Remembering that we are primarily considering the case when $D = 0$ allows us to equate the right side to simply $S$. Similarly, the middle term is equivalent to $\frac{2S + 1}{3}$. Clearly, then, this inequality simplifies to the true inequality $\frac{S+1}{3} < \frac{2S+1}{3} <  S$. 

It is now important to consider the cases where $p$ is not zero. In Figure \ref{fig:exp_dist} the results seen in Lemmas II.1 - II.5 are combined to see the values of the expected distance functions as a function of $p$, the number of elements swapped.

\begin{figure}[h!]
\caption{Expected Distances Between Items of a Graph with 2 Categories of Size 20}
\vspace{0.4cm}
\includegraphics[width=0.5\textwidth]{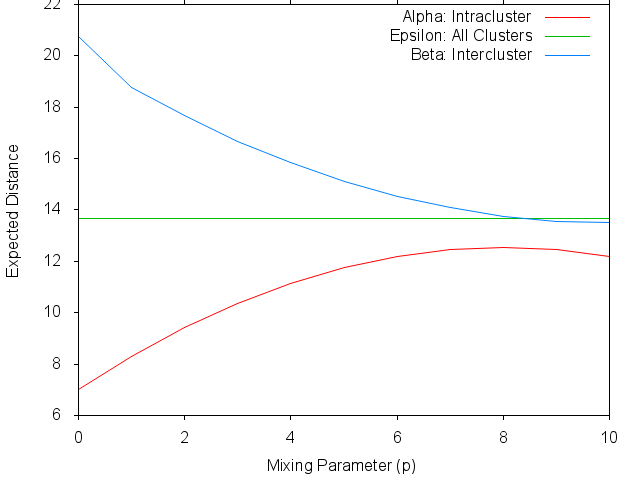}
\centering
\label{fig:exp_dist}
\end{figure}

As $p$ increases, the expected distance between elements of different categories decreases and the expected distance between elements of the same category increases. It is only once $p$ gets large enough ($p > 8$) that the expected distance between elements of different categories dips beneath the threshold. Regardless of the value of $p$, the expected intracategory distance is lower than the intercategory distance. In the next section, we introduce the Weighted Label Propagation algorithm which is used to examine the community structure on the created graph.

\section{Weighted Label Propagation Algorithm}

In this section we introduce a modification of the traditional label propagation algorithm seen in \cite{2007_Label_Prop}. As with the original, this algorithm starts with each node having its own unique label. At each iteration, a node's label is updated to be the most common label of all its neighboring nodes, with ties broken uniformly and randomly. As iterations continue, most of the labels disappear as many nodes take on the same label. The convergence point in this algorithm is when the label of every node does not change from iteration to iteration. This is equivalent to every node having the most common label of all of its neighbors. At the conclusion, all nodes sharing the same label are grouped into communities.

In the original version, all neighboring labels have equal weight regardless of their location in the graph. The distance of a label from its source vertex is not considered at all. In the weighted version of the algorithm, we incorporate this distance function in order to allow labels that occur close to their source label to have more weight than labels that are very far from their source. This modification serves to better localize the clustering. A secondary benefit of this modification is that it prevents the entire graph from being classified as one community purely due to an increase in the frequency of a label. 

As stated, this algorithm is very similar to the non-weighted version except that before assigning the label with the largest count, the label counts are re-weighted based on their distance from their source vertex. This weighting should be a function of the distance, $d$, and can be done in a variety of ways. In this paper, two such distance functions chosen were the linear $W_1(d) = \frac{1}{d}$ and the exponential $W_2(d) = \frac{1}{2^{d}}$.

\begin{figure*}
\centering
\begin{minipage}{.5\textwidth}
  \centering
  \includegraphics[width=.7\linewidth]{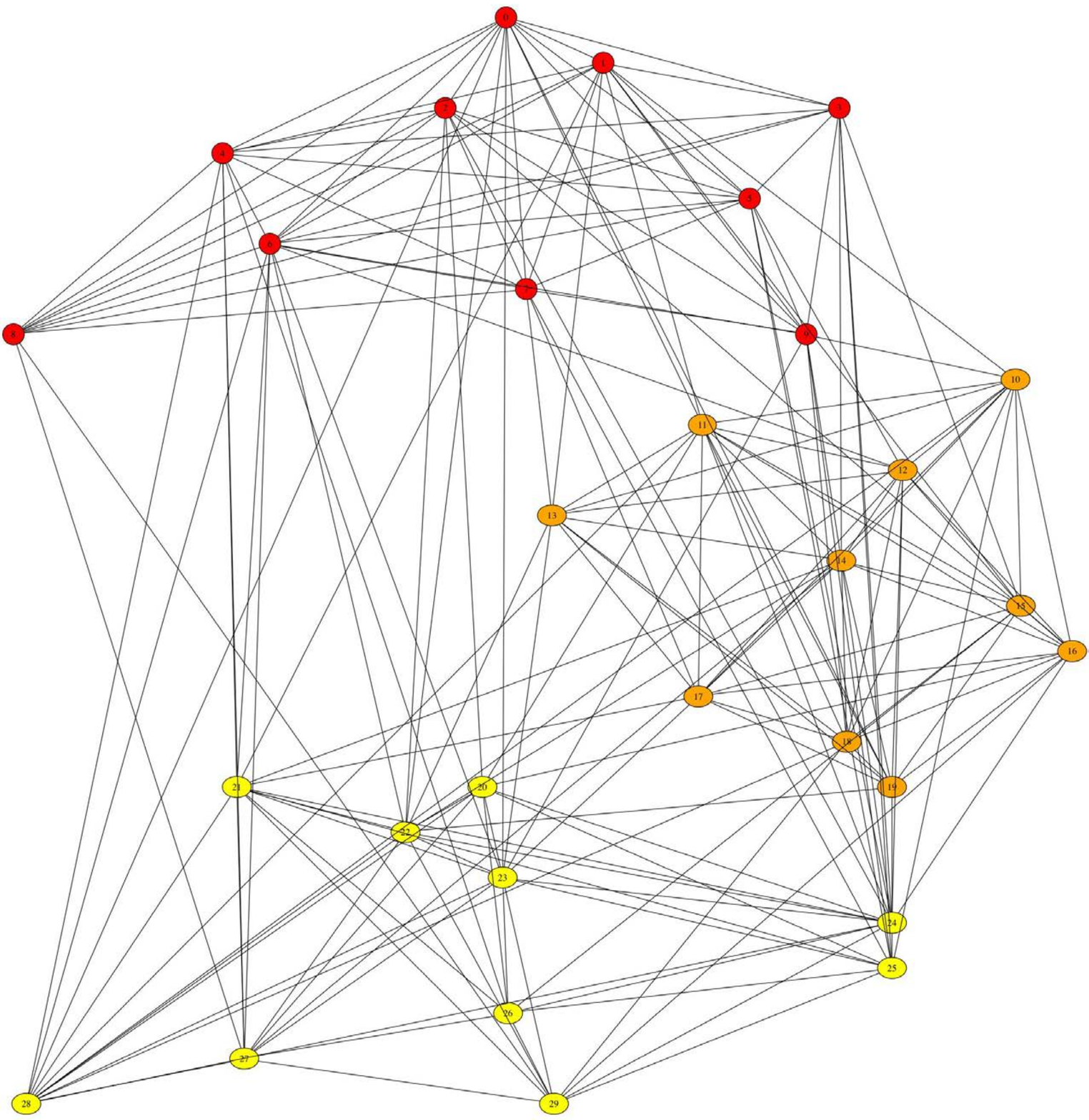}
  \caption{Weighted Label Propagation Results}
  \label{fig:true}
\end{minipage}%
\begin{minipage}{.5\textwidth}
  \centering
  \includegraphics[width=.7\linewidth]{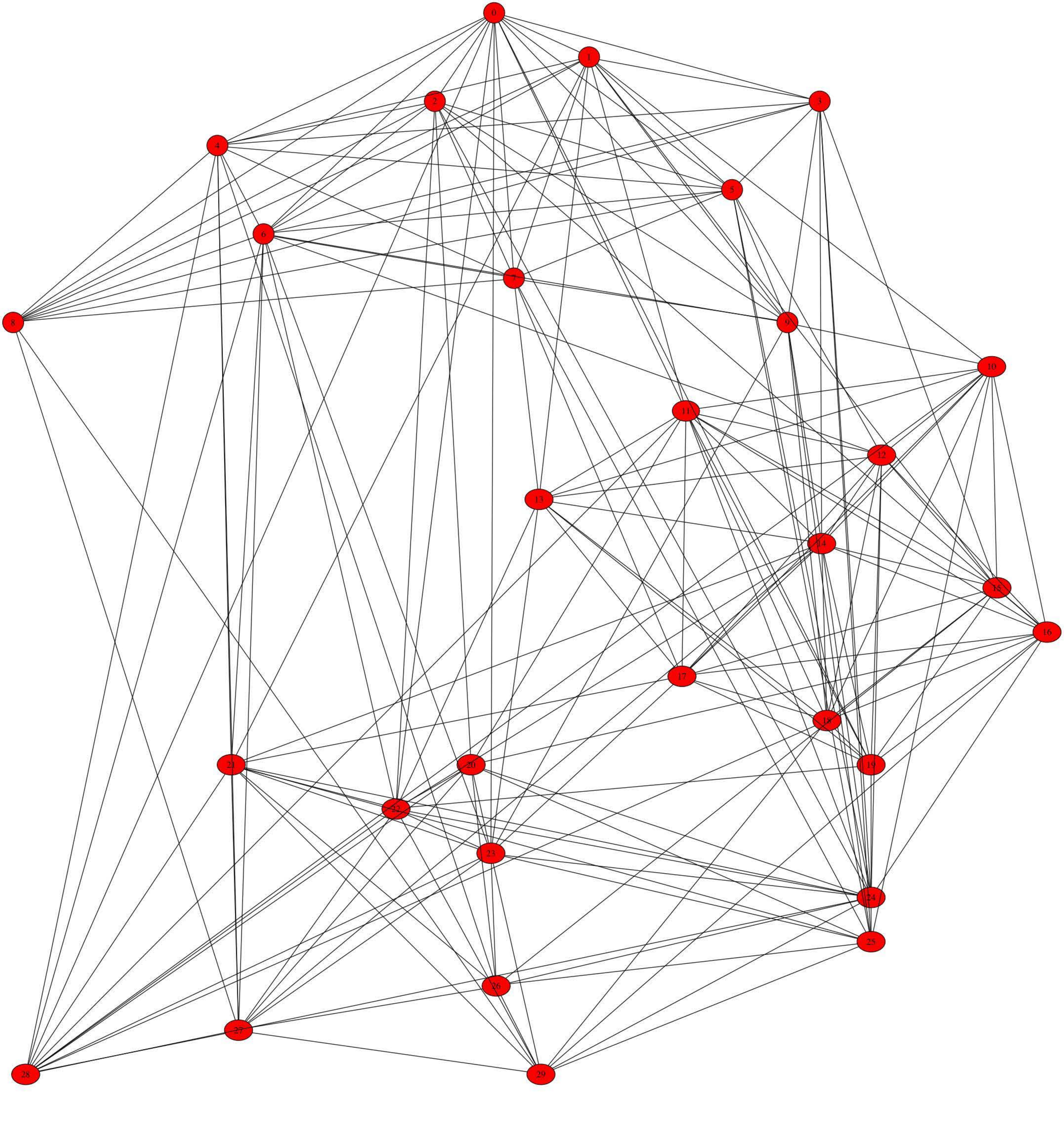}
  \caption{Label Propagation Results}
  \label{fig:guess}
\end{minipage}
\caption{Here we can see the results of the algorithm on the categorization of 30 elements split into 3 categories. The left graph shows the true categorization while the right graph shoes the categorization resulting from the proposed algorithm. As can be seen, the results are nearly identical except for one orange node that is categorized as yellow.}
\label{fig:compare_true_guess}
\end{figure*}

\begin{algorithm}[H]
  \caption{Weighted Label Propagation}
  \begin{algorithmic}
   \label{alg:WeightedLabelProp}
  \STATE \textbf{Input: } A graph, $\G{G}$. A weighting function, $W : \mathbb{R} \mapsto [0,1]$. 
  \STATE \textbf{Initialize: } $t = 0$, $\forall v \in \G{G}$, $L_v(t) = v$. Create a random ordering of the nodes, $X$. 
  \WHILE{labels have not converged}
	\FOR{$v \in X$}
		\STATE Initialize a map, $C$, of all zero values to represent the counts of each label
		\FOR{$z \in N_{v}$}
			\STATE Let $L = L_z(t)$
			\STATE $C[L] = C[L] + W(dist(L, z))$
		\ENDFOR
		\STATE $L_v(t + 1) = \underset{z \in N_{v}}{\operatorname{argmax}} \:\: C[L_z(t)]$ 
	\ENDFOR
	\STATE $t = t+1$
 \ENDWHILE
 \STATE Split the vertices into communities based on common labels
  \end{algorithmic}
\end{algorithm}

The weighted algorithm can be performed in either a synchronous or asynchronous manner. The synchronous manner is the one described above and the asynchronous version differs only in the update step. Rather than always updating based on the label from the previous time step, the asynchronous version uses the new label of any nodes that have already been updated at the current time step. 

The time complexity of this algorithm is also linear per iteration. The only difference is that it requires a preprocessing step to find the distances from all vertices to all other vertices in the graph. This is done using Dijkstra's algorithm for each vertex $O(V)*O(E \log{V}) = O(VE \log{V})$. However, since the graphs we are considering are relatively sparse, we have $E = O(V)$, so this reduces to $O(V^2 \log{V})$. We have omitted the analysis of the convergence of the weighted label propagation algorithm and intend to explore it in future work. Please see Appendix C for a comparison of the weighted label propagation algorithm with other algorithms.

\section{Experimental Results}

\begin{figure}[ht!]
\vspace{0.4cm}
\includegraphics[scale=0.4]{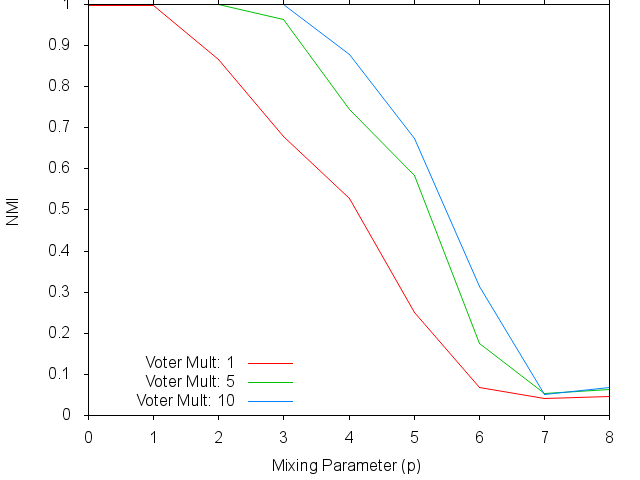}
\centering
\caption{NMI Accuracies vs Mixing Parameter}
\label{fig:nmi}
\end{figure}

\subsection{Synthetic Data}
\begin{figure*}[ht!]
\centering
\begin{minipage}{.5\textwidth}
  \centering
  \includegraphics[width=.7\linewidth]{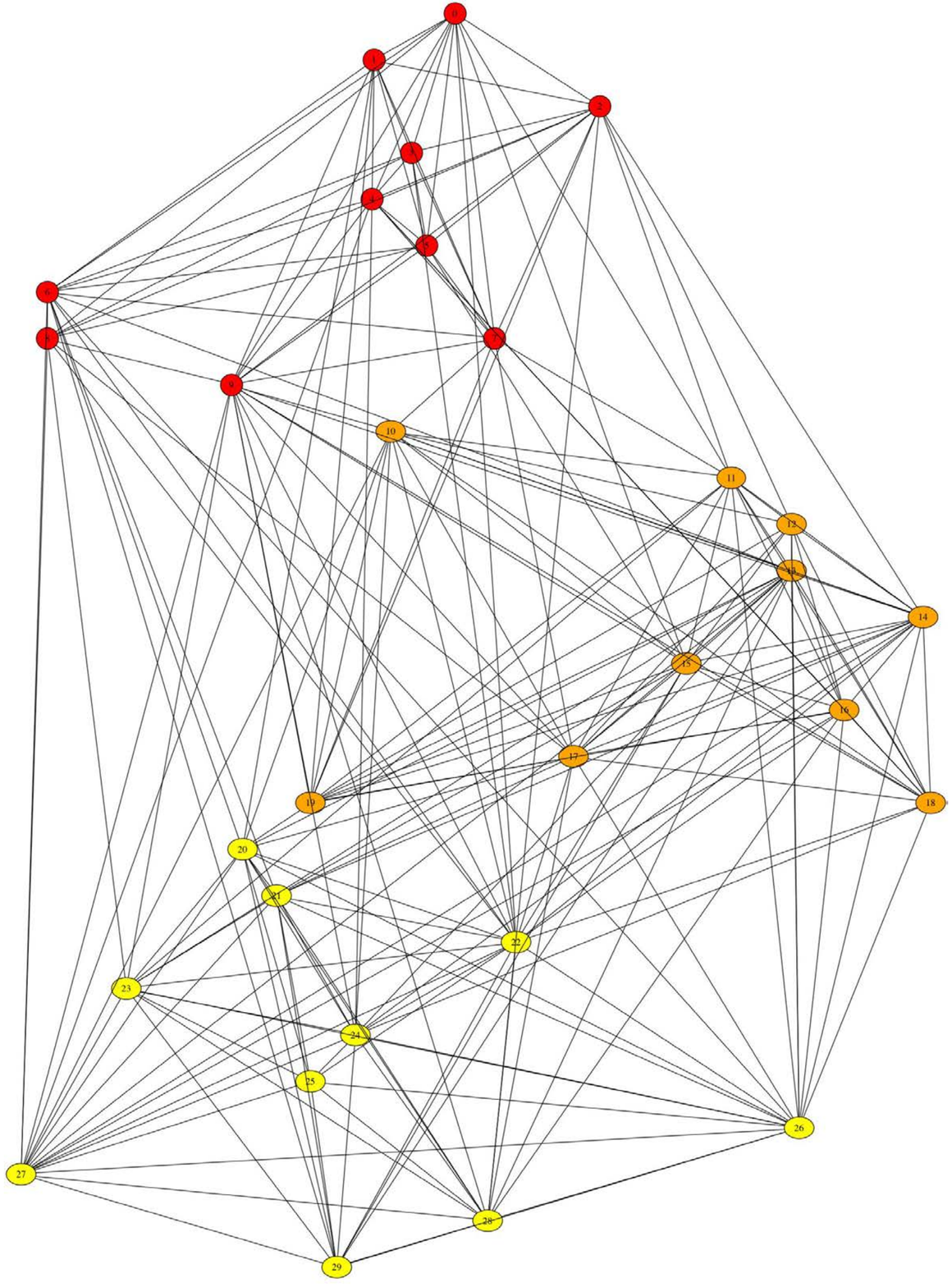}
  \caption{True Categorization}
  \label{fig:true}
\end{minipage}%
\begin{minipage}{.5\textwidth}
  \centering
  \includegraphics[width=.7\linewidth]{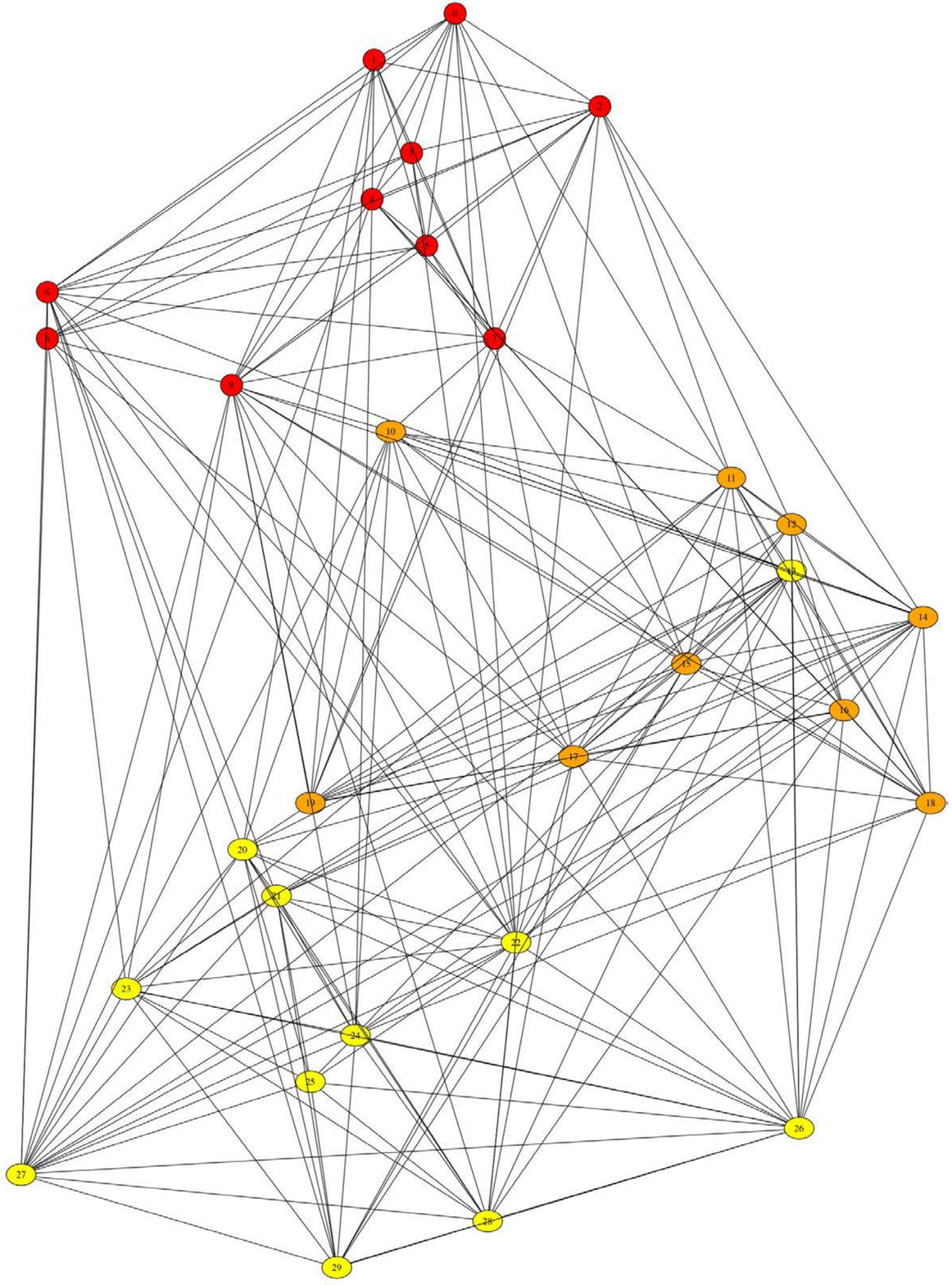}
  \caption{Categorization Resulting from Algorithm}
  \label{fig:guess}
\end{minipage}
\caption{Here we can see the results of the algorithm on the categorization of 30 elements split into 3 categories. The left graph shows the true categorization while the right graph shoes the categorization resulting from the proposed algorithm. As can be seen, the results are nearly identical except for one orange node that is categorized as yellow.}
\label{fig:compare_true_guess}
\end{figure*}

In order to judge the success of the algorithm with the synthetic data, we compared the resulting community labeling to the initial categorization from the beginning of the algorithm. The metric we used to judge this comparison was normalized mutual information (NMI) \cite{NMI}. This metric ranges from $0$ to $1$, where a $1$ signifies that we have a perfect match with the true community structure and a $0$ signifies no relationship with the true structure.

\begin{definition} \emph{Normalized Mutual Information} \\
\[{\tt NMI}(A, B) = \frac{2 I(A,B)}{H(A) + H(B)}\]
\emph{where $A$ and $B$ are two partitionings of elements, $I$ denotes the mutual information between the partitions and $H$ denotes the entropy of a partition.}
\end{definition}

We can examine Figure \ref{fig:nmi} to see the effects of the voting multiplier and the mixing parameter on the accuracy of the model. Here the voting multiplier signifies the ratio of the number of voters to the total number of elements. As we would expect, the larger the voting multiplier, the higher the accuracy, regardless of the mixing parameter. An increase in the mixing parameter does show a general trend of decreasing the NMI with the true categorization.

\subsection{Real Data}

In order to test the algorithm, we used the MovieLens data \cite{MovieLens}, which consists of user ratings on a variety of movies. Users rate the movies from 0-5 and may leave some of the movies unranked. We explored different thresholds for cutting off the similarity function and present the results below. In every situation, the entire dataset chosen consists of the union of all the categories.

\begin{itemize}
\item \textbf{Comparing Amityville Horror Movies and Kid's Movies \\
$\mathbf{\epsilon} = 0.94 \rightarrow \mbox{Edge Ratio} = 0.440$} 
\begin{enumerate}
\item \emph{Category 1:} Toy Story (1995), Lion King, The (1994), Aladdin (1992), Snow White and the Seven Dwarfs (1937), Alice in Wonderland (1951)
\item \emph {Category 2:} Aladdin and the King of Thieves (1996), Jungle Book, The (1994), Pocahontas (1995)
\item \emph{Category 3:} Amityville 1992: It's About Time (1992), Amityville 3-D (1983), Amityville: A New Generation (1993), Amityville II: The Possession (1982), Amityville Horror, The (1979), Amityville Curse, The (1990)
\end{enumerate}

\item \textbf{Comparing Amityville Horror Movies and Star Trek Movies \\
$\mathbf{\epsilon} = 0.94 \rightarrow \mbox{Edge Ratio} = 0.359$} 
\begin{enumerate}
\item \emph{Category 1:} Star Trek VI: The Undiscovered Country (1991), Star Trek: The Wrath of Khan (1982), Star Trek III: The Search for Spock (1984), Star Trek IV: The Voyage Home (1986), Star Trek: Generations (1994), Star Trek: The Motion Picture (1979)
\item \emph {Category 2:} Star Trek V: The Final Frontier (1989), Amityville 1992: It's About Time (1992), Amityville 3-D (1983), Amityville: A New Generation (1993), Amityville II: The Possession (1982), Amityville Horror, The (1979), Amityville Curse, The (1990)
\end{enumerate}

\item \textbf{Comparing Star Wars and Star Trek Movies \\
$\mathbf{\epsilon} = 0.92 \rightarrow \mbox{Edge Ratio} = 0.266$} 
\begin{enumerate}
\item \emph{Category 1:} Star Wars (1977), Empire Strikes Back, The (1980), Return of the Jedi (1983)
\item \emph {Category 2:} Star Trek VI: The Undiscovered Country (1991), Star Trek: The Wrath of Khan (1982), Star Trek III: The Search for Spock (1984), Star Trek IV: The Voyage Home (1986), Star Trek: Generations (1994), Star Trek: The Motion Picture (1979)
\item \emph {Category 3:} Star Trek V: The Final Frontier (1989)
\end{enumerate}

\textbf{$\mathbf{\epsilon} = 0.915 \rightarrow \mbox{Edge Ratio} = 0.333$} 
\begin{enumerate}
\item \emph{Category 1}: Star Wars (1977), Empire Strikes Back, The (1980), Return of the Jedi (1983), Star Trek: The Wrath of Khan (1982) 
\item \emph{Category 2}: Star Trek VI: The Undiscovered Country (1991), Star Trek III: The Search for Spock (1984), Star Trek IV: The Voyage Home (1986), Star Trek: Generations (1994), Star Trek: The Motion Picture (1979), Star Trek V: The Final Frontier (1989)
\end{enumerate}
\end{itemize}

We are able to form a categorization of the movies chosen by examining the relationship between the ratings of all of the voters. In each of the above situations, a clear distinction can be seen between the groups of movies we are examining. Although at some points there is an additional group introduced, the groups reflect the separation that is implied based on the chosen films. 

\section{Appendix}

\subsection{Calculation of Expected Value of Similarity Function:}

We are calculating the expected value of the similarity function \[Sim(a, b, N) = 1 - \frac{|a - b|}{N}\]
\begin{align}
\mathbb{E}[S(a,b,N)] &= \mathbb{E}\left[1 - \frac{|a - b|}{N}\right] \nonumber \\
& = \mathbb{E}\left[1\right] - \mathbb{E}\left[\frac{|a - b|}{N}\right] 
=  1 - \frac{1}{N} \mathbb{E}\left[|a - b|\right] \\ 
\mathbb{E}\left[|a - b|\right] &= \frac{1}{\binom{N}{2}} \sum\limits_{i = 1}^{N-1} i(N-i) \\
& = \frac{2}{N(N-1)}\frac{(N-1)N(N+1)}{6} 
= \frac{N+1}{3} \\
\mathbb{E}[S(a,b,N)] &= 1 - \frac{1}{N}\frac{(N+1)}{3} = 1 - \frac{N+1}{3N}
\end{align}

\subsection{Calculation of Expected Intercategory Distance}

Below we present the formulation of the expected intercategory distance with $p$ swaps. The proof of the equation with $1$ swap can be seen by simply substituting $p = 1$. Similarly to the derivation of the intracategory distance, the intercategory distance was derived by examining the possible combinations of the elements that are going to be compared. 

\begin{align}
&\mathbb{E}\left[|a - b|_{diff}\right] = \nonumber \\
& \frac{\left(\binom{S - 1}{p}^2 + \binom{S - 1}{p - 1}^2 \right) S^3 (D + 1) + 4 \binom{S - 2}{p - 1} \binom{S}{p} \sum\limits_{i = 1}^{S - 1} i (S - i)}{S^2 * \binom{S}{p}^2}
\end{align}

$\left(\binom{S - 1}{p}^2 + \binom{S - 1}{p - 1}^2 \right) S^3 (D + 1)$: This represents the total sum of distances when the two elements are in separate categories. $S^3 (D + 1)$ is the cumulative distance for all possible pairs of elements in each arrangement of this type and $\left(\binom{S - 1}{p}^2 + \binom{S - 1}{p - 1}^2 \right)$ is the total number of times we will have comparisons between elements of different categories. The first term represents arrangements in which neither of the elements swaps from its own category. We can think of this as fixing the element we are considering and picking the $p$ elements to swap from the other $S - 1$ elements of the category for both of the categories. The second term represents arrangements when the elements actually swap with each other and neither is in its true category. We can calculate the number of ways of doing this as requiring the element to swap and picking the other $p - 1$ elements from the $S - 1$ remaining elements of the category for each category. 

$4 \binom{S - 2}{p - 1} \binom{S}{p} \sum\limits_{i = 1}^{S - 1} i (S - i)$: This term represents the total sum of distances when the two elements are in the same category. We know that $\sum\limits_{i = 1}^{S - 1} i (S - i)$ is the total distance for each occurrence of this comparison (as seen in Appendix A). Here $4 \binom{S - 2}{p - 1} \binom{S}{p}$ represents the arrangement of categories in which this comparison arises. We know that we are then fixing two elements from the category (1 to swap and 1 to stay), so we then pick the other $p - 1$ elements to swap from the remaining $S - 2$ elements. Now we know we will have an occurrence of this type regardless of what happens in the other category, leading to the factor of $\binom{S}{p}$ (the number of possible outcomes from the second category). We then multiply by 2 twice: once because either element can be the one that stays and once because we must consider this summation of distances for both of the categories.

Finally, we see the denominator is formed from the product of the number of possible swaps at each arrangement. We know that there are $\binom{S}{p}^2$ possible arrangements resulting from each cluster picking $p$ elements to swap. Similarly, we know that each arrangement has $S^2$ swaps just by simple examination.

\subsection{Comparison of Weighted Label Propagation with Other Algorithms:}

In order to judge the success of the Weighted Label Propagation algorithm, we compared its success in determining the community structure of a series of planted partition models:

\begin{table}[h!]
$\textbf{p} = 0.7, \: \textbf{q} = 0.01, \: \mbox{\textbf{number categories}} = 10, \: \mbox{\textbf{size}} = 5$\\

\begin{tabular}{| l | c | c | c |}
\hline
& CNM & Regular & Weighted \\ \hline
Avg Num of Categories: & 8.81 & 9.35 & 9.99 \\ 
Avg NMI with Truth: & 0.9568 & 0.9764 & 0.9940 \\ 
Avg Modularity: & 0.6998 & 0.6992 & 0.7020 \\ \hline
\end{tabular}
\end{table}

\begin{table}[h!]
$\textbf{p} = 0.75, \: \textbf{q} = 0.01, \: \mbox{\textbf{number categories}} = 10, \: \mbox{\textbf{size}} = 5$\\

\begin{tabular}{| l | c | c | c |}
\hline
& CNM & Regular & Weighted \\ \hline
Avg Num of Categories: & 8.97 & 9.29 & 10.02 \\ 
Avg NMI with Truth: & 0.9644 & 0.9754& 0.9970 \\ 
Avg Modularity: & 0.7003 & 0.6969 & 0.7024 \\ \hline
\end{tabular}
\end{table}

\begin{table}[h!]
$\textbf{p} = 0.8, \: \textbf{q} = 0.01, \: \mbox{\textbf{number categories}} = 10, \: \mbox{\textbf{size}} = 5$\\

\begin{tabular} {| l | c | c | c |}
\hline
& CNM & Regular & Weighted \\ \hline
Avg Num of Categories: & 8.99 & 9.61 & 10.01 \\ 
Avg NMI with Truth: & 0.9646 & 0.9873 & 0.9976 \\ 
Avg Modularity: & 0.7066 & 0.7069 & 0.7089 \\ \hline
\end{tabular}
\end{table}

As is seen in all of the above examples, the weighted label propagation algorithm performs better than the CNM algorithm and the regular label propagation algorithm in terms of the correct number of categories found, the normalized mutual information (NMI) and the modularity of the resulting partition.

\section{Acknowledgements}
This research was supported by NSF Research Experiences for Undergraduate (REU) program via the grant NSF:CCF:1319653. 

\bibliographystyle{IEEEtran}
\bibliography{community_detection}

%

\end{document}